\documentclass{article}
\usepackage{ijcai23}

\usepackage{times}
\usepackage{soul}
\usepackage{url}
\usepackage[hidelinks]{hyperref}
\usepackage[utf8]{inputenc}
\usepackage[small]{caption}
\usepackage{graphicx}
\usepackage{amsmath}
\usepackage{enumerate}
\usepackage{amsthm}
\usepackage{booktabs}
\usepackage{algorithm}
\usepackage{algorithmic}
\usepackage{bm}
\usepackage{soul}
\usepackage[hidelinks]{hyperref}
\usepackage[utf8]{inputenc}
\usepackage{tikz}
\tikzset{>=stealth}
\usepackage[switch]{lineno}
\usepackage{epsfig}
\usepackage{amsfonts}
\usepackage{multirow}
\usepackage{enumitem}
\theoremstyle{plain}
\newtheorem{theorem}{Theorem}

\theoremstyle{definition}

\theoremstyle{remark}

\usepackage[symbol]{footmisc}

\def \zerov {\bm{0}}
\def \onev {\bm{1}}

\def \bv {\bm{b}}

\def \ev {\bm{e}}
\def \gv {\bm{g}}

\def \wv {\bm{w}}
\def \xv {\bm{x}}

\def \zv {\bm{z}}
\def \hv {\bm{h}}
\def \alphav {\bm{\alpha}}

\def \Hv {\mathbf{H}}

\def \Hv {\mathbf{H}}
\def \Iv {\mathbf{I}}
\def \Kv {\mathbf{K}}
\def \Qv {\mathbf{Q}}
\def \Xv {\mathbf{X}}
\def \Yv {\mathbf{Y}}

\def \alphav {\bm{\alpha}}
\def \betav {\bm{\beta}}
\def \xiv {\bm{\xi}}
\def \gammav {\bm{\gamma}}

\def \zetav {\bm{\zeta}}
\def \epsilonv {\bm{\epsilon}}

\def \Ccal {\mathcal{C}}
\def \Dcal {\mathcal{D}}

\def \Ical {\mathcal{I}}

\def \Scal {\mathcal{S}}

\def \Xcal {\mathcal{X}}
\def \Rcal {\mathcal{R}}
\def \Ycal {\mathcal{Y}}

\def \Hbb {\mathbb{H}}
\def \Ibb {\mathbb{I}}
\def \Rbb {\mathbb{R}}

\def \dt {\widetilde{d}}
\def \pt {\widetilde{p}}

\def \Qvt {\widetilde{\Qv}}
\def \alphavt {\widetilde{\alphav}}
\def \zetavt {\widetilde{\zetav}}
\def \betavt {\widetilde{\betav}}
\def \gammavt {\widetilde{\gammav}}
\def \Dcalt {\widetilde{\Dcal}}
\def \odm {\textit{Optimal margin Distribution Machine}}

\DeclareMathOperator*{\argmin}{argmin}

\def \st {\mathrm{s.t.}}
\def \diag {\mathrm{diag}}

\renewcommand{\thefootnote}{\fnsymbol{footnote}}

\title{Scalable Optimal Margin Distribution Machine}

\author{
Yilin Wang \and Nan Cao \and Teng Zhang\footnote{Contact Author} \and Xuanhua Shi \And Hai Jin\\
\affiliations
National Engineering Research Center for Big Data Technology and System\\
Services Computing Technology and System Lab, Cluster and Grid Computing Lab\\
School of Computer Science and Technology, Huazhong University of Science and Technology, China\\
\emails
\{yilin\_wang, nan\_cao, tengzhang, xhshi, hjin\}@hust.edu.cn
}

\begin{document}

\maketitle
\renewcommand{\thefootnote}{\arabic{footnote}}

\begin{abstract}
    \odm~(ODM) is a newly proposed statistical learning framework rooting in the latest margin theory, which demonstrates better generalization performance than the traditional large margin based counterparts. However, it suffers from the ubiquitous scalability problem regarding both computation time and memory storage as other kernel methods. This paper proposes a scalable ODM, which can achieve nearly ten times speedup compared to the original ODM training method. For nonlinear kernels, we put forward a novel distribution-aware partition method to make the local ODM trained on each partition be close and converge fast to the global one. When linear kernel is applied, we extend a communication efficient SVRG method to accelerate the training further. Extensive empirical studies validate that our proposed method is highly computational efficient and almost never worsen the generalization.
\end{abstract}

\section{Introduction}

Recently, the study on margin theory~\cite{Gao2013On} demonstrates an upper bound disclosing that maximizing the minimum margin does not necessarily result in a good performance. Instead, the distribution rather than a single margin is much more critical. Later on, the study on lower bound~\cite{Allan2019Margin} further proves that the upper bound is almost optimal up to a logarithmic factor. Inspired by these insightful works, \citeauthor{Zhang2019Optimal}~\shortcite{Zhang2019Optimal} propose the \odm~(ODM), which explicitly optimizes the margin distribution by maximizing the mean and minimizing the variance simultaneously and exhibits much better generalization than the traditional large margin based counterparts. Due to the superiority shown on both binary and multi-class classification tasks, many works attempt to extend ODM to more genreal learning settings, just to list a few, cost-sensitive learning~\cite{Zhou2016Large,Cheng2017Large}, weakly supervised learning~\cite{Zhang2018Optimal,Zhang2018Semi,Luan2020Optimal,Zhang2020Instance,Cao2022Positive}, multi-label learning~\cite{Tan2020Multi,Cao2021Partial}, online learning~\cite{Zhang2020Dynamic}, and regression~\cite{Rastogi2020Large}. Plenty of successes on various learning tasks validate the superiority of this new statistical learning framework. However, with the dramatic progress of digital technologies, the data generated devices become as diverse as computers, mobile phones, smartwatches, cars, etc., and the amount of data created each day grows tremendously, thus these ODM based extensions suffer from the scalability problem regarding both computation time and memory storage as other kernel methods. 

There have been many works devoted to accelerating kernel methods, which can be roughly classified into three categories. The first category is based on approximation, e.g., the random Fourier feature~\cite{Rahimi2007Random} takes the trigonometric functions as basis functions to approximate the kernel mapping, the Nyström method~\cite{Williams2001Using} generates a low-rank approximations by sampling a subset of columns, and the coreset~\cite{Tan2019Coreset} adaptively sketches the whole data by choosing some landmark points. The second category divides the data into partitions on which local models are trained and combined to produce a larger local or global model, e.g., in~\cite{Graf2004cascade,Hsieh2014dc,Singh2017Dip}, a tree architecture on partitions is designed first, guided by which the solutions of different partitions are aggregated; in~\cite{Yu2005Making,Vazquez2006dsvm,loosli2007training}, the key instances identification and exchange are further introduced to accelerate the training; in~\cite{Si2017MEKA}, both low-rank and clustering structure of the kernel matrix are taken into account to get an approximation of kernel matrix. The third category is directly applying the distributed-style optimization method, such as the augmented Lagrangian method~\cite{Forero2010Admm} and the alternating direction method of multipliers~\cite{Boyd2011Admm}, or extending existing solver to a distributed environment, e.g., distributed SMO~\cite{Cao2006SMO}.

Notice that the random Fourier feature adopts a data-independent kernel mapping and the Nyström method takes a data distribution-unaware sampling, hence their performance are both inferior to the coreset method~\cite{Tan2019Coreset}, which inspires us to leverage data as heavily as possible. Moreover, the distributed off-the-shelf \textit{quadratic programming}~(QP) solvers can be directly applied to train ODM, but they are all general approaches thus ignore the intrinsic structure of the problem and can hardly achieve the greatest efficiency. To take the best of both worlds, this paper proposes a specially designed \textit{scalable ODM}~(SODM). Specifically, we put forward a novel data partition method so that ODM trained on each partition has a solution close to that trained on the whole data. When some partitions are merged to form a larger partition, the solution on it can be quickly obtained by concatenating the previous local solutions as the initial point. Besides, in the case of the linear kernel, we extend a communication efficient SVRG method to accelerate the training further. To summarize, the remarkable differences of SODM compared with existing scalable QP solvers are threefold:

\begin{enumerate}
    \item SODM incorporates a novel partition strategy, which makes the local ODM on each partition be close to the global one so that the training can be accelerated.
    \item SODM accelerates the training further when the linear kernel is applied by extending a communication efficient SVRG.
    \item SODM achieves nearly ten times speedup meanwhile, maintain ODM's generalization performance in most situations.
\end{enumerate}

The rest of this paper is organized as follows. We first introduce some preliminaries, and then present the technical detail of our method. After that we show the experimental results and empirical observations. Finally we conclude the paper with future work.

\section{Preliminaries}

Throughout the paper, scalars are denoted by normal case letters (e.g., $m$ and $M$). Vectors and matrices are denoted by boldface lower and upper case letters, respectively (e.g., $\xv$ and $\Xv$). The $(i,j)$-th entry of matrix $\Xv$ is $[\Xv]_{ij}$. Sets are designated by upper case letters with mathcal font (e.g., $\Scal$). The input space is $\Xcal \subseteq \Rbb^N$ and $\Ycal = \{1,-1\}$ is the label set. For any positive integer $M$, the set of integers $\{ 1, \ldots, M \}$ is denoted by $[M]$. For the feature mapping $\phi: \Xcal \mapsto \Hbb$ associated to some positive definite kernel $\kappa$ where $\Hbb$ is the corresponding \textit{reproducing kernel Hilbert space}~(RKHS), $\kappa(\xv, \zv) = \langle \phi(\xv), \phi(\zv) \rangle_{\Hbb}$ holds for any $\xv$ and $\zv$.

\subsection{Optimal Margin Distribution Machine}

The traditional large margin based methods maximize the minimum margin, and the obtained decision boundary is only determined by a small number of instances with the minimum margin~\cite{Scholkopf2001Learning}, which may hurt the generalization performance.

On the other hand, ODM explicitly optimizes the margin distribution. Given a labeled data set $\{ (\xv_i, y_i) \}_{i \in [M]}$, ODM is formalized by maximizing the margin mean and minimizing the margin variance simultaneously:
\begin{align*}
    \min_{\wv, \xi_i, \epsilon_i} & ~ p(\wv) = \frac{1}{2} \|\wv\|^2 + \frac{\lambda}{2M} \sum_{i \in [M]} \frac{\xi_i^2 + \upsilon \epsilon_i^2}{(1 - \theta)^2} \\
    \st                           & ~ 1 - \theta - \xi_i \le y_i \wv^\top \phi(\xv_i) \le 1 + \theta + \epsilon_i, ~ \forall i \in [M],
\end{align*}
where the margin mean has been fixed as 1 since scaling $\wv$ does not affect the decision boundary, the hyperparameter $\lambda$ is to balance the regularization and empirical loss, the hyperparameter $\upsilon$ is for trading-off the two different kinds of deviation from margin mean, and the hyperparameter $\theta$ is introduced to tolerate small deviations no more than $\theta$. 

By introducing the Lagrange multipliers $\zetav, \betav \in \Rbb^M_+$ for the $2 M$ inequality constraints respectively, the dual problem of ODM is
\begin{align}
     & \min_{\zetav, \betav \in \Rbb^M_+} d(\zetav, \betav) = \frac{1}{2} (\zetav - \betav)^\top \Qv (\zetav - \betav) + \frac{M c}{2} (\upsilon \|\zetav\|^2 \nonumber \\
    \label{eq: odm-dual-1}
     & \qquad + \|\betav\|^2) + (\theta-1) \onev_M^\top \zetav + (\theta+1) \onev_M^\top \betav,
\end{align}
where $[\Qv]_{ij} = y_i y_j \kappa (\xv_i, \xv_j)$ and $c = (1 - \theta)^2 / \lambda \upsilon$ is a constant. By denoting $\alphav = [\zetav; \betav]$, the dual ODM can be rewritten as a standard convex QP problem:
\begin{align} \label{eq: odm-dual-2}
    \min_{\alphav \in \Rbb^{2M}_+} ~ f(\alphav) = \frac{1}{2} \alphav^\top
    \Hv \alphav + \bv^\top \alphav,
\end{align}
where
\begin{align*}
    \Hv = \begin{bmatrix}
        \Qv + M c \upsilon \Iv & -\Qv          \\
        - \Qv                  & \Qv + M c \Iv
    \end{bmatrix}, \quad \bv = \begin{bmatrix}
        (\theta - 1) \onev_M \\
        (\theta + 1) \onev_M
    \end{bmatrix}.
\end{align*}
Notice that Eqn.~\eqref{eq: odm-dual-2} only involves $2M$ decoupled box constraints $\alphav \succeq \zerov$, thus it can be efficiently solved by a dual coordinate descent method~\cite{Zhang2019Optimal}. To be specific, in each iteration, only one variable is selected to update while other variables are kept as constants, which yields the following univariate QP problem of $t$:
\begin{align} \label{eq: dcd}
    \min_t ~ f(\alphav + t \ev_i) = \frac{1}{2} [\Hv]_{ii} t^2 + [\nabla f(\alphav)]_i t + f(\alphav),
\end{align}
with a closed-form solution $\max ([\alphav]_i - [\nabla f(\alphav)]_i / [\Hv]_{ii}, 0)$.

\section{Proposed Method}

SODM works in distributed data level, i.e., dividing the data into partitions on which local models are trained and used to find the larger local or global models. For simplicity, we assume initially there are $K = p^L$ partitions with the same cardinality $m$, i.e., $m = M / K$. The data set $\{ (\xv_i, y_i) \}_{i \in [M]}$ are ordered so that the first $m$ instances are on the first partition, and the second $m$ instances are on the second partition, etc. That is for any instance $(\xv_i, y_i)$, the index of partition to which it belongs is $P(i) = \lceil i/m \rceil$ where $\lceil \cdot \rceil$ is the ceil function. 

Suppose $\{(\xv_i^{(k)}, y_i^{(k)})\}_{i \in [m]}$ is the data of the $k$-th partition, the local ODM trained on it is [cf. Eqn.~\eqref{eq: odm-dual-1}]
\begin{align*}
     & \min_{\zetav_k, \betav_k \in \Rbb^m_+} d_k(\zetav_k, \betav_k) = \frac{1}{2} (\zetav_k - \betav_k)^\top \Qv^{(k)} (\zetav_k - \betav_k) \\
     & + \frac{m c}{2} ( \upsilon \|\zetav_k\|^2 + \|\betav_k\|^2) + (\theta - 1) \onev_m^\top \zetav_k + (\theta + 1) \onev_m^\top \betav_k,
\end{align*}
where $[\Qv^{(k)}]_{ij} = y_i^{(k)} y_j^{(k)} \kappa (\xv_i^{(k)}, \xv_j^{(k)})$. This problem can be rewritten as a standard convex QP problem in the same manner as Eqn.~\eqref{eq: odm-dual-2}, and efficiently solved by dual coordinate descent method as Eqn.~\eqref{eq: dcd}.

\begin{algorithm}[ht]
\caption{SODM}
\label{Alg:1}
    \textbf{Input}: Data set $\Dcal = {\{(\xv_i, y_i)\}_{i \in [M]}}$, partition control parameter $p$, number of stratums $S$, number of iterations $L$.\\
    \textbf{Output}: The dual solution$.$
    \begin{algorithmic}[1]
        \STATE Initialize $S$ stratums $\Ccal_1, \ldots, \Ccal_S$ by Eqn.~\eqref{landmarks}-\eqref{clusterstrategy}.
        \STATE Initialize partitions $\Dcal_1, \ldots, \Dcal_{p^L}$ by sampling without replacement from stratums $\Ccal_1, \ldots, \Ccal_S$. 
        \STATE Initialize $\alphav_1, \ldots, \alphav_{p^L}$ as $\zerov$.
        \FOR{$l = L, \ldots, 1$}
        \IF{all $\alphav_1, \ldots, \alphav_{p^l}$ converge}
            \STATE \textbf{return} $[\alphav_1; \ldots; \alphav_{p^l}]$.
        \ENDIF
        \FOR{$k = 1,\ldots, p^l$}
        \STATE Solve the local ODM on $\Dcal_k$ by dual coordinate descent with $\alphav_k$ as the initial solution.
        \IF{$k \equiv 0 ~ (\mathrm{mod} ~ p)$}
            \STATE Form new $\Dcal_{k/p}$ by merging $\Dcal_{k-p+1}, \ldots, \Dcal_k$.
            \STATE $\alphav_{k/p} = [\alphav_{k-p+1}; \ldots; \alphav_k]$.
        \ENDIF
        \ENDFOR
        \ENDFOR
        \STATE \textbf{return} $[\alphav_1; \ldots; \alphav_p]$.
    \end{algorithmic}
\end{algorithm}

Once the parallel training of $p^L$ local ODMs are completed, we get $p$ solutions. Then we merge every $p$ partitions to form $K / p = p^{L-1}$ larger partitions. On each larger partition, a new local ODM is trained again by dual coordinate descent method, but the optimization procedure is not executed from the scratch. Instead, the previous $p$ solutions are concatenated as the initial point of the optimization. By our proposed novel partition strategy in Section \ref{sec: partition-strategy}, this concatenated solution is already a good approximation to the optimal solution thus converges much faster. The above procedure is repeated until the solution converges or all the partitions are merged together. Algorithm~\ref{Alg:1} summarizes the pseudo-code of SODM.

\subsection{Convergence}

In this section, we present a theorem to guarantee the convergence of the proposed method. Notice that the optimization variables on each partition are decoupled, they can be jointly optimized by the following problem [cf. Eqn.~\eqref{eq: odm-dual-1}]
\begin{align}
    \min_{\zetav, \betav \in \Rbb^M_+} & \dt (\zetav, \betav) = \frac{1}{2} (\zetav - \betav)^\top \Qvt (\zetav - \betav) + \frac{m c}{2} ( \upsilon \|\zetav\|^2 \nonumber \\
    \label{eq: approximation-problem}
                                       & + \|\betav\|^2) + (\theta-1) \onev_M^\top \zetav + (\theta+1) \onev_M^\top \betav,
\end{align}
where $\Qvt = \diag(\Qv^{(1)}, \ldots, \Qv^{(K)})$ is a block diagonal matrix. It can be seen that the smaller the $K$, the more close the Eqn.~\eqref{eq: approximation-problem} to ODM, and when $K=1$, it exactly degenerates to ODM. Therefore, SODM deals with ODM by solving a series of problems which approaches to it, and the solution of former problems can be helpful for the optimization of the latter ones.

\begin{theorem}\label{theorem-1}
    Suppose the optimal solutions of ODM and its approximate problem, i.e., Eqn.~\eqref{eq: approximation-problem}, are $\alphav^\star = [\zetav^\star; \betav^\star]$ and $\alphavt^\star = [\zetavt^\star; \betavt^\star]$, respectively, then the gaps between these two optimal solutions satisfy
    \begin{align}
         \label{eq: thm1-1}
         & 0 \le d(\zetavt^\star, \betavt^\star) - d(\zetav^\star, \betav^\star) \le U^2 (Q + M (M - m) c), \\
         \label{eq: thm1-2}
         & \| \alphavt^\star - \alphav^\star\|^2 \le \frac{U^2}{M c \upsilon} (Q + M (M - m) c),
    \end{align}
    where $U = \max ( \|\alphav^\star\|_\infty, \|\alphavt^\star\|_\infty)$ upperbounds the infinity norm of solutions, and $Q = \sum_{i,j:P(i) \neq P(j)} |[\Qv]_{ij}|$ is the sum of the absolute values of $\Qv$'s entries which turn to zero in $\Qvt$.
\end{theorem}

Due to the page limitations, we only provide the sketch of proof here. The full proof can be found in \nameref{appendixproof}.

\begin{proof}
    The left-hand side of the Eqn.~\eqref{eq: thm1-1} is due to the optimality of $\zetav^\star$ and $\betav^\star$.

    By comparing the definition of $d(\zetav, \betav)$ in Eqn.~\eqref{eq: odm-dual-1} and $\dt (\zetav, \betav)$ in Eqn.~\eqref{eq: approximation-problem}, we can find that the only differences are the change of $\Qv$ to $\Qvt$ and $M$ to $m$. Therefore the gap between $d(\zetav^\star, \betav^\star)$ and $\dt (\zetav^\star, \betav^\star)$ can be upper bounded by $U$ and $Q$. The gap between $d(\zetavt^\star, \betavt^\star)$ and $\dt (\zetavt^\star, \betavt^\star)$ can be upper bounded in the same manner. Combining these together with $\dt (\zetavt^\star, \betavt^\star) \le \dt (\zetav^\star, \betav^\star)$ can yield the right-hand side of the Eqn.~\eqref{eq: thm1-1}.

    Notice that $f(\alphavt^\star)$ is a quadratic function, hence besides the gradient $\gv$ and Hessian matrix $\Hv$, all its higher derivatives vanish, and it can be precisely expanded at $\alphav^\star$ as
    \begin{align*}
        f(\alphav^\star) + \gv^\top (\alphavt^\star - \alphav^\star) + \frac{1}{2} (\alphavt^\star - \alphav^\star)^\top \Hv (\alphavt^\star - \alphav^\star),
    \end{align*}
    in which $\gv^\top (\alphavt^\star - \alphav^\star)$ is nonnegative according to the the first order optimality condition. Furthermore, $\Hv$ can be lower bounded by the sum of a positive semidefinite matrix and a scalar matrix:
    \begin{align*}
        \Hv \succeq \begin{bmatrix}
            \Qv   & -\Qv \\
            - \Qv & \Qv
        \end{bmatrix} + M c \upsilon \begin{bmatrix}
            \Iv    \\
             & \Iv
        \end{bmatrix}.
    \end{align*}
    By putting all these together, we can show that $\| \alphavt^\star - \alphav^\star\|^2$ is upper bounded by $f(\alphavt^\star) - f(\alphav^\star)$, i.e., $d(\zetavt^\star, \betavt^\star) - d(\zetav^\star, \betav^\star)$, and with the right-hand side of the Eqn.~\eqref{eq: thm1-1}, we can derive the Eqn.~\eqref{eq: thm1-2}.
\end{proof}

This theorem indicates that the gap between the optimal solutions and the suboptimal solutions obtained in each iteration depends on $M-m$ and $Q$. As the iteration going on, the partitions become larger and larger, then the number of instances $m$ on each partition approaches to the total number of instances $M$; on the other hand, the matrix $\Qvt$ approaches to $\Qv$ which makes $Q$ decrease. Therefore, the solution obtained in each iteration of SODM is getting closer and closer to that of ODM, that is to say, our proposed algorithm converges.

\subsection{Partition Strategy} \label{sec: partition-strategy}

In this section we detail the partition strategy. It can significantly affect the optimization efficiency thus plays a more important role in our proposed method. Up to now, most partition strategies utilize the clustering algorithms to form the partitions. For example, \citeauthor{Hsieh2014dc}~\shortcite{Hsieh2014dc} regards each cluster of the kernel $k$-means as a partition. However, ODM heavily depends on the mean and variance of the training data. Directly treating clusters as partitions will lead to huge difference between the distribution of each partition and the whole data, and consequently huge gap between the local solutions and global solution. 

To preserve the original distribution possibly, we borrow the idea from stratified sampling, i.e., we first divide the data set into some homogeneous stratums, and then apply random sampling within each stratum. To be specific, suppose the goal is to generate $K$ partitions. We first choose $S$ landmark points $\{ \phi(\zv_s) \}_{s \in [S]}$ in RKHS, and then construct one stratum for each landmark point by assigning the rest of instances to the stratum in which its nearest landmark point lies, i.e., the index of stratum containing $\xv_i$ is
\begin{align} \label{landmarks}
    \varphi(i) = \argmin_{s \in [S]} \| \phi(\xv_i) - \phi(\zv_s) \|.
\end{align}
For each stratum $\Ccal_s$, we equally divide it into $K$ pieces by random sampling without replacement and take one piece from each stratum to form a partition, hence totally $K$ partitions are created.

The remaining question is how to select these landmark points. Obviously, they should be representative enough to sketch the whole data distribution. To this end, we exploit the minimal principal angle between different stratum:
\begin{align*}
    \tau = \min_{i \neq j} \left\{ \arccos \frac{\langle \phi(\xv), \phi(\zv) \rangle}{\|\phi(\xv)\|\|\phi(\zv)\|} ~ \bigg| ~ \xv \in \Ccal_i, \zv \in \Ccal_j \right\}.
\end{align*}
Apparently, the larger the angle, the higher variation among the stratums, and the more representative each partition is, which is strictly described by the following theorem.

\begin{theorem}\label{theorem-2}
    For shift-invariant kernel $\kappa$ with $\kappa(\xv, \zv) = \kappa(\xv - \zv)$, assume $\kappa(0) = r^2$, that is $\|\phi(\xv)\| = r$ for any $\xv$. With the partition strategy described above, we have
    \begin{align*}
        d_k & (\zetav_k, \betav_k) - d(\zetav^\star, \betav^\star) 
        \leq U^2 M^2 c +2 U M \\ & \quad + \frac{U^2}{2} ( M^2 r^2 + r^2 \cos \tau (2C - M^2) ), ~ \forall k \in [K],
    \end{align*}
    where $C = \sum_{i,j \in [M]} 1_{\varphi(i) \neq \varphi(j)}$, and $U$ is the same with Theorem~\ref{theorem-1}.
\end{theorem}

\begin{proof}
    We construct the auxiliary data set $\widetilde{\Dcal}_k$ by repeating each instance in $\Dcal_k$ for $K$ times, and then show that primal ODM on $\widetilde{\Dcal}_k$ and $\Dcal_k$ have the same optimal objective. Since the strong duality theorem holds for ODM, we have $d_k(\zetav_k, \betav_k) = p_k(\wv_k) = \widetilde{p}_k(\wv) = \dt_k (\zetavt_k, \betavt_k)$. Next we decompose $\dt_k (\zetavt_k, \betavt_k) - d(\zetav^\star, \betav^\star)$ into
    \begin{align*}
        \frac{1}{2} (\zetavt_k - \betavt_k)^\top \Qvt_k (\zetavt_k - \betavt_k) - \frac{1}{2} (\zetav^\star - \betav^\star)^\top \Qv (\zetav^\star - \betav^\star),
    \end{align*}
    and
    \begin{align*}
        \frac{M c \upsilon}{2} & (\|\zetavt_k\|^2 - \|\zetav^\star\|^2) + \frac{M c}{2} (\|\betavt_k\|^2 - \|\betav^\star\|^2) \\
        & + (\theta - 1) \onev_M^\top (\zetavt_k - \zetav^\star) + (\theta + 1) \onev_M^\top (\betavt_k - \betav^\star).
    \end{align*}
    Putting the upper bounds of these two terms together can conclude the proof.
\end{proof}

In this theorem, we derive an upper bound of the gap between the optimal objective value on $\Dcal$ and $\Dcal_k$. Notice that $2C > M^2$ holds for any $s \in [S]$ when $|\Ccal_s| < M/2$ is satisfied, a quite mild condition, thus we can get more approximate solution in each partition by maximizing the minimal principal angle $\tau$ in RKHS.

Unfortunately, the resultant maximization problem is difficult to solve, so we can hardly acquire the optimal landmark points. But notice that the Gram matrix formed by landmark points should be diagonally dominant and the more strict the better, we can resort to maximizing its determinant. Specifically, suppose $\zv_1,...,\zv_s$ are given, we seek $\zv_{s+1}$ to maximize
\begin{align*}
    \begin{vmatrix}
    \Kv_{s,s} & \Kv_{s,s+1} \\ \Kv_{s,s+1}^\top & \kappa(\zv_{s+1}, \zv_{s+1})
    \end{vmatrix} = r^2 (r^2 - \Kv_{s,s+1}^\top \Kv_{s,s}^{-1} \Kv_{s,s+1}),
\end{align*}
where $\Kv_{s,s} \in \Rbb^{s \times s}$ is the Gram matrix formed by $\zv_1,...,\zv_s$, and $\Kv_{s,s+1} = [\kappa(\zv_{s+1}, \zv_1); ...; \kappa(\zv_{s+1}, \zv_s)]$ is a column vector. The equality holds due to the Schur's complement. As for $\zv_1$, since any choice makes no difference, we can directly set it as $\xv_1$, and generate other landmark points iteratively via
\begin{align} \label{clusterstrategy}
    \zv_{s+1} = \argmin_{\zv_{s+1}} ~ \Kv_{s,s+1}^\top \Kv_{s,s}^{-1} \Kv_{s,s+1}, ~ \forall s \in [S-1].
\end{align}

It is noteworthy that each partition generated by our proposed strategy extracts proportional instances from each stratum, thus preserves the distribution. Besides, compared with other partition strategies based on $k$-means~\cite{Singh2017Dip}, we consider both the original feature space and the situation when data can hardly be linearly separated. Last but not least, our partition strategy is computationally efficient.

\begin{algorithm}[!h]
\caption{Accelerated SODM for linear kernel}
\label{Alg:2}
    \textbf{Input}: Data set $\Dcal = {\{(\xv_i, y_i)\}_{i \in [M]}}$, number of partitions $K$, number of stratums $S$, number of epoch $E$, step size $\eta$.\\
    \textbf{Output}: Solution $\wv^{(E)}$ at epoch $E.$
    \begin{algorithmic}[1]
        \STATE Initialize $S$ stratums $\Ccal_1, \ldots, \Ccal_S$ by Eqn.~\eqref{landmarks}-\eqref{clusterstrategy}.
        \STATE Initialize partitions $\Dcal_1, \ldots, \Dcal_{K}$ by sampling without replacement from stratums $\Ccal_1, \ldots, \Ccal_S$. 
        \STATE Generate the auxiliary array $\Rcal_1, \ldots, \Rcal_K$ where $\Rcal_j = \{i \mid (\xv_i, y_i) \in \Dcal_j\}$.
        \FOR{$l = 0, 1, \ldots, E-1$}
        \STATE The center node sends $\wv^{(l)}$ to each node.
        \FOR{each node $j = 1, 2, \ldots, K$ in parallel}
        \STATE $\hv_j^{(l)} = \sum_{i \in {\Dcal_j}} {\nabla {p_i} (\wv^{(l)})}$.
        \ENDFOR
        \STATE The center node computes $\hv^{(l)} = \frac{1}{M} \sum_{j = 1}^K {\hv_j^{(l)}}$ and sends it to each node.
        \STATE $\wv_0^{(l+1)} = \wv^{(l)}$.
        \STATE $t=0$.
        \FOR{$j = 1, 2, \ldots, K$}
        \STATE Sample instances $(\xv_i, y_i)$ from $\Dcal_j$ where $i \in \Rcal_j$.
        \STATE $\wv_{t + 1}^{(l+1)} = {\wv_t^{(l+1)}} - \eta (\nabla {p_i} (\wv_t^{(l+1)}) - \nabla {p_i} (\wv^{(l)}) + \hv^{(l)})$.
        \STATE $t = t+1$.
        \STATE $\Rcal_j = \Rcal_j \backslash i$.
        \IF{$\Rcal_j = \emptyset $}
        \STATE Continue.
        \ENDIF
        \ENDFOR
        \STATE ${\wv^{(l+1)}} = \wv_t^{(l+1)}$.
        \ENDFOR
        \STATE \textbf{return} $\wv^{(E)}$.
    \end{algorithmic}
\end{algorithm}

\subsection{Acceleration for Linear Kernel}

Dual coordinate descent method requires too many computation and storage resources, mainly caused by the enormous kernel matrix. But fortunately, when linear kernel is applied, we can directly solve the primal form of ODM, avoiding the computation and storage of kernel matrix. 

The objective function of ODM is differentiable and the gradient of $p(\wv)$ on instance $(\xv_i, y_i)$ is
\begin{align*}
    \nabla {p_i} (\wv) 
    & = \wv + \frac{\lambda(y_i \wv^\top \xv_i + \theta - 1)y_i \xv_i 1_{i \in \Ical_1}}{(1-\theta)^2} \\
    & \quad + \frac{\lambda\upsilon(y_i \wv^\top \xv_i - \theta - 1)y_i \xv_i 1_{i \in \Ical_2}}{(1-\theta)^2},
\end{align*}
where $\Ical_1 = \{i \mid y_i \wv^\top \xv_i < 1 - \theta\}$ and $\Ical_2 = \{i \mid y_i \wv^\top \xv_i > 1 + \theta\}$. \textit{Distributed SVRG}~(DSVRG)~\cite{Lee2017DSVRG} can be exploited in this scenario. It generates a series of extra auxiliary data sets sampling from the original data set without replacement which share the same distribution as the whole data set, so that an unbiased estimation of the gradient can be acquired. In each iteration, all nodes (partitions) are joined together to compute the full gradient first. Then each node performs the iterative update of SVRG in serial in a ``round robin'' fashion, i.e., let all nodes stay idle except one node performing a certain steps of iterative updates using its local auxiliary data and passing the solution to the next node. Algorithm~\ref{Alg:2} summarizes the process of DSVRG for SODM.

\begin{table*}[!h]
    \centering
    \begin{tabular}{ l | c c c c c c c c }
        \hline \hline \noalign{\smallskip}
        Data sets  & \textsf{gisette} & \textsf{svmguide1}  & \textsf{phishing} & \textsf{a7a}  & \textsf{cod-rna}  & \textsf{ijcnn1} & \textsf{skin-nonskin} & \textsf{SUSY} \\
        \noalign{\smallskip} \hline \noalign{\smallskip}
        \#Instance & 7,000 & 7,089  & 11,055  & 32,561  & 59,535  & 141,691  & 245,057  & 5,000,000 
        \\
        \#Feature  & 5,000 & 4  & 68  & 123  & 8  & 22  & 3  & 18 
        \\
        \noalign{\smallskip} \hline \hline
    \end{tabular}
    \caption{Data set statistics}
    \label{tab:1}
\end{table*}

\begin{table*}[!h]
    \centering
    \begin{tabular}{ l | c | c | c | c | c | c | c | c | c}
        \hline \hline \noalign{\smallskip}
        \multirow{2}{*}{Data sets} &  ODM & \multicolumn{2}{c|}{Ca-ODM} & \multicolumn{2}{c|}{DiP-ODM}&  \multicolumn{2}{c|}{DC-ODM} & \multicolumn{2}{c}{SODM}  \\
        \noalign{\smallskip} \cline{2-10} \noalign{\smallskip}
        ~           & Acc. & Acc.  & Time & Acc. & Time & Acc. & Time & Acc. & Time\\
        \noalign{\smallskip} \hline \noalign{\smallskip}
        \textsf{gisette}  & .976 & .957 & 90.22    & .970 & 68.02 & .964 & 70.44 & \textbf{.972} & 59.89 
        \\
        \textsf{svmguide1}  & .970 & .872 & 38.90  & .903  & 35.25 & .943 & 50.11 & \textbf{.944} & 28.74
        \\
        \textsf{phishing}  & .941 & .880  & 49.60  & .901  & 52.61 & .936 & 59.47  & \textbf{.938}  & 25.22 
        \\
        \textsf{a7a}  & .882 & .824  & 68.36  & .813  & 61.24 & .815 & 106.51 & \textbf{.838}  & 32.67 
        \\
        \textsf{cod-rna}  & N/A & .892  & 499.38  & .905  & 532.68 & .931 & 400.61 & \textbf{.933}  & 55.41 
        \\
        \textsf{ijcnn1}  & N/A & .889  & 185.20  & .893  & 182.71 & .915 & 226.26 & \textbf{.927}  & 40.32 
        \\
        \textsf{skin-nonskin}  & N/A & .806  & 338.73  & .830  & 437.20 & \textbf{.962} & 407.46 & .956  & 283.36 
        \\
        \textsf{SUSY}  & N/A & .733  & 4280.23  & .744  & 5678.66 & .747 & 7009.36 & \textbf{.760}  & 1004.33 
        \\
        \noalign{\smallskip} \hline \hline
    \end{tabular}
    \caption{The test accuracy and time cost (in seconds) of different methods using RBF kernel. The best accuracy on each data set is bolded. N/A means the corresponding method does not return results in 48 hours.}
    \label{tab:2}
\end{table*}

\section{Experiments}

In this section, we evaluate the proposed algorithms by comparing with other SOTA scalable QP solvers.

\begin{table*}[!h]
    \centering
    \begin{tabular}{ l | c | c | c | c | c | c | c | c | c }
        \hline \hline \noalign{\smallskip}
        \multirow{2}*{Data sets} &  \multicolumn{1}{c|}{ODM} &
        \multicolumn{2}{c|}{Ca-ODM} & \multicolumn{2}{c|}{DiP-ODM} & \multicolumn{2}{c|}{DC-ODM}  & \multicolumn{2}{c}{SODM}  \\
        \noalign{\smallskip} \cline{2-10} \noalign{\smallskip}
        ~           & Acc. & Acc.  & Time & Acc. & Time & Acc. & Time & Acc. & Time\\
        \noalign{\smallskip} \hline \noalign{\smallskip}
        \textsf{gisette} & .972 & .953 & 82.35 & .966 & 74.36 & \textbf{.968} & 66.32 & \textbf{.968} & 28.57 
        \\
        \textsf{svmguide1} & .964 & .863 & 35.27 & .898 & 40.52 & \textbf{.933} & 41.85 & .931 & 18.93 
        \\
        \textsf{phishing}  & .937 & .894  & 33.84  & .921  & 38.60 & .926 & 29.04 & \textbf{.933} & 11.75
        \\
        \textsf{a7a}  & .850 & .795  & 47.59  & .831  & 59.17& .833 & 85.42 & \textbf{.847} & 16.41  
        \\
        \textsf{cod-rna}  & .938 & .882  & 435.19  & .894  & 434.77& .890 & 331.46 & \textbf{.934} & 17.29 
        \\
        \textsf{ijcnn1} & .913 & .896  & 228.43  & .903  & 208.81 & .883 & 214.66  & \textbf{.920} & 21.15 
        \\
        \textsf{skin-nonskin}  & .917 & .796  & 158.12  & .903  & 256.78& \textbf{.922} & 340.30 & .909 & 21.15 
        \\
        \textsf{SUSY}  & .774  & .734  & 3790.37  & .738  & 3829.23 & .747 & 7095.32 & \textbf{.760} & 178.92
        \\
        \noalign{\smallskip} \hline \hline
    \end{tabular}
    \caption{The test accuracy and time cost (in seconds) of different methods using linear kernel. The best accuracy on each data set is bolded.}
    \label{tab:3}
\end{table*}

\begin{figure*}[!h]
    \begin{center}
    \includegraphics[width=0.95\textwidth]{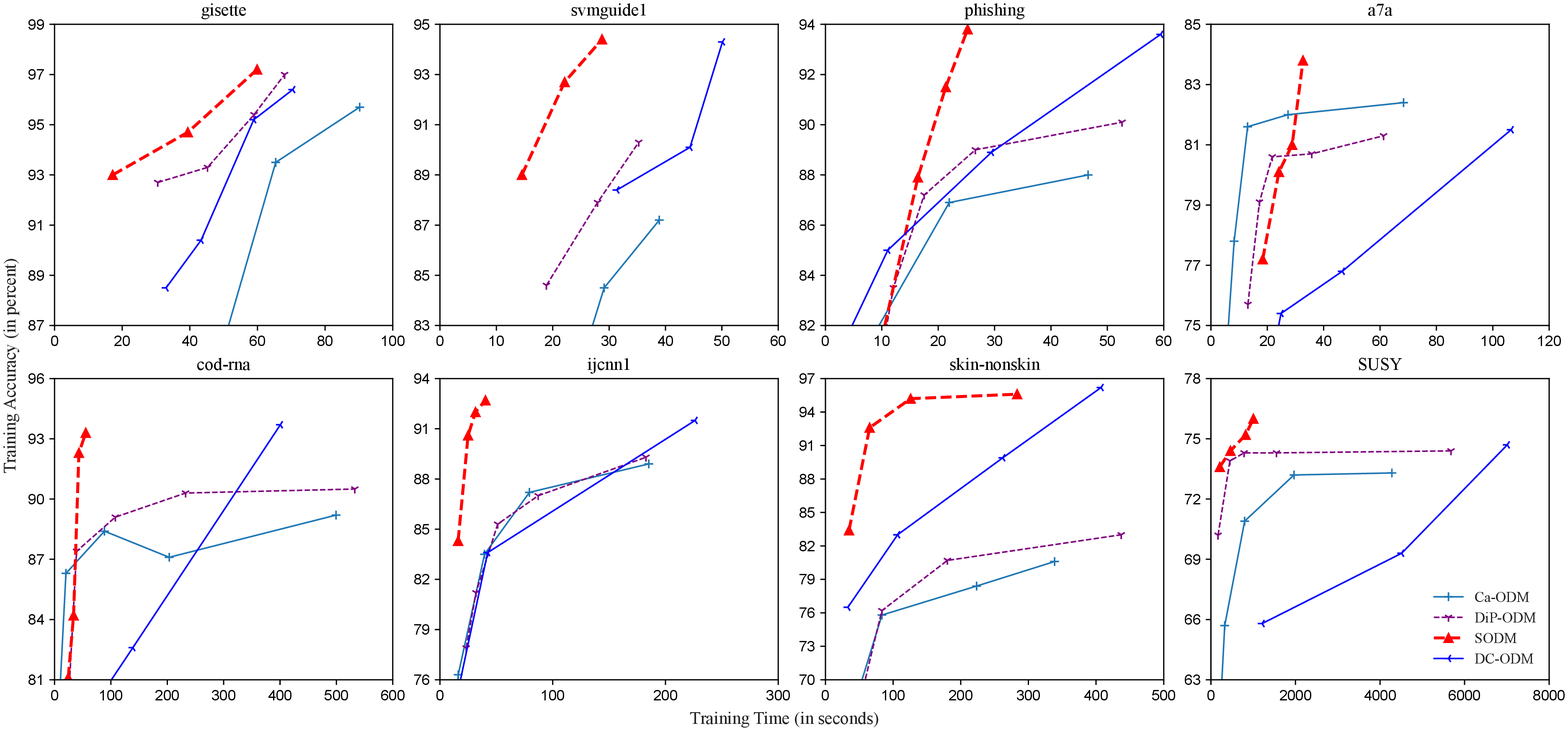}
    \caption{Comparisons of different methods using RBF kernel. Each point indicates the result when stop at different levels.}
    \label{fig:1}
    \end{center}
\end{figure*}

\begin{figure}[!h]
    \centering
    \includegraphics[width=0.95\linewidth]{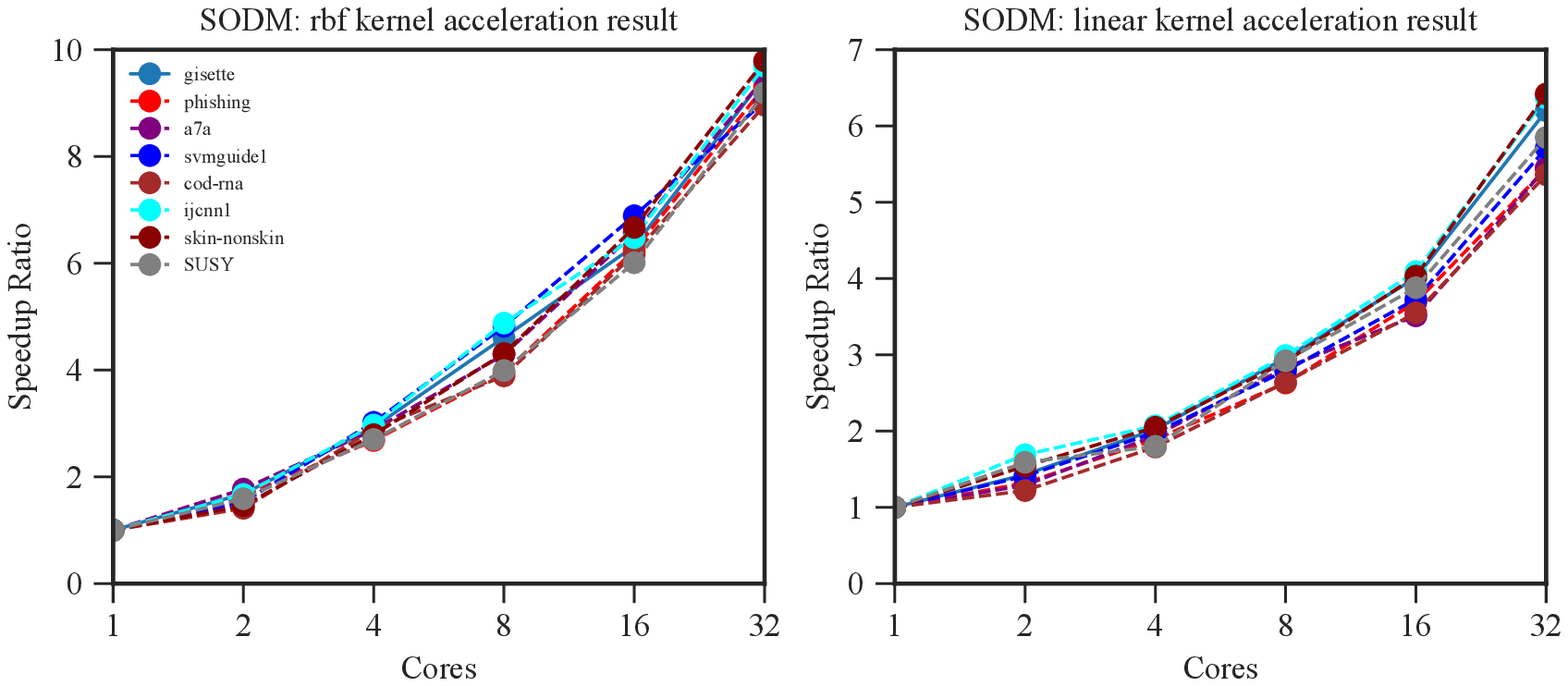}
    \caption{Training speedup ratio with cores increasing from 1 to 32 for SODM}
    \label{fig:3}
\end{figure}

\begin{figure*}[!h]
    \centering
    \includegraphics[width=0.95\textwidth]{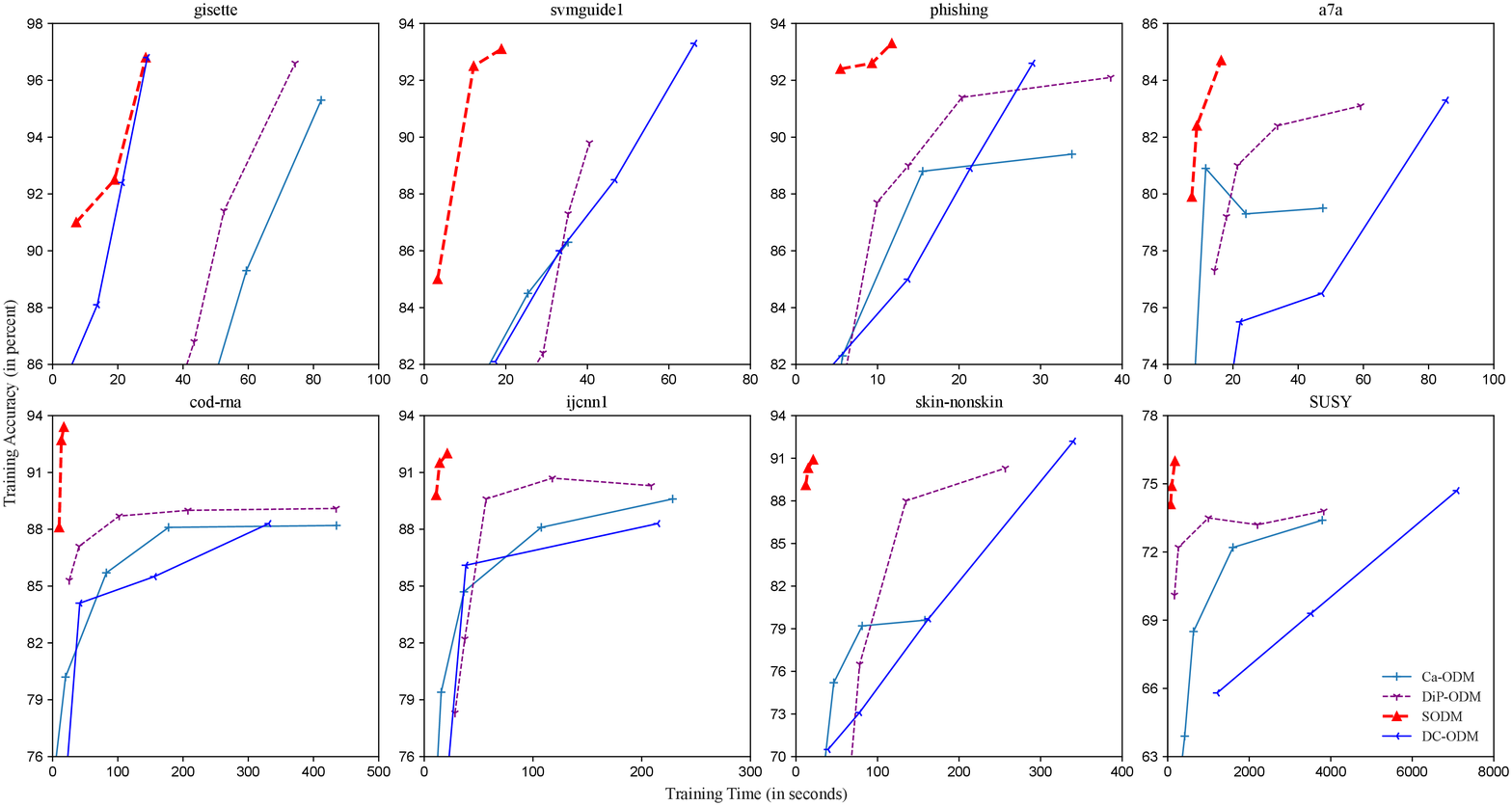}
    \caption{Comparisons of different methods using linear kernel. Each point of SODM indicates the result when every one third of epochs executed. Other points indicate the result stop at different levels.}
    \label{fig:2}
\end{figure*}

\begin{figure*}[!h]
    \centering
    \includegraphics[width=0.95\textwidth]{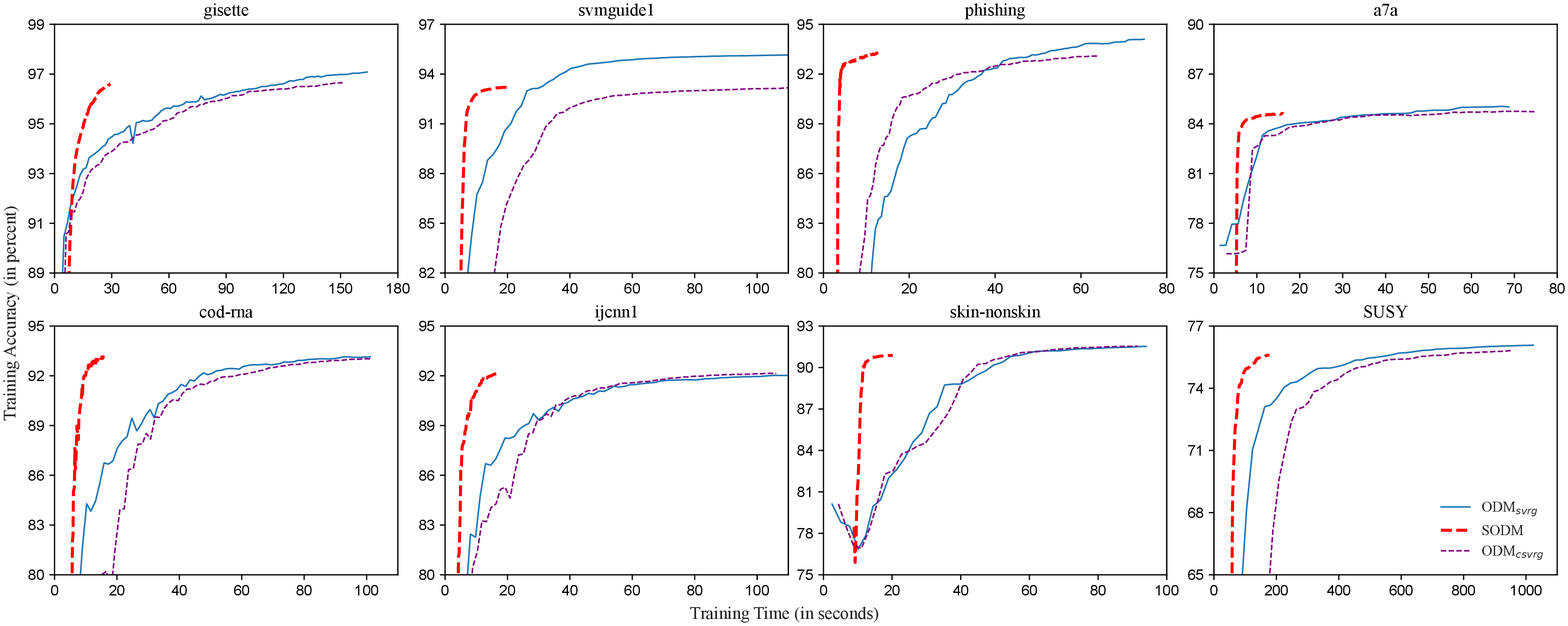}
    \caption{Comparisons of different gradient based methods}
    \label{fig:4}
\end{figure*}

\subsection{Setup}

All the experiments are performed on eight real-world data sets. The statistics of these data sets are summarized in Table~\ref{tab:1}. All features are normalized into the interval $[0,1]$. For each data set, eighty percent of instances are randomly selected as training data, while the rest are testing data. All the experiments are performed on a Spark~\cite{Zaharia2012Spark} cluster with one master and five workers. Each machine is equipped with 16 Intel Xeon E5-2670 CPU cores and 64GB RAM. Our implementation are available on Github ~\footnote{\url{https://github.com/CGCL-codes/SODM}}.

SODM is compared with three SOTA scalable QP solvers, i.e., Cascade approach~(Ca-ODM)~\cite{Graf2004cascade}, DiP approach~(DiP-ODM)~\cite{Singh2017Dip}, and DC approach~(DC-ODM)~\cite{Hsieh2014dc}. Besides, to evaluate the efficiency of the accelerated SODM for linear kernel, two SOTA gradient based methods are implementd, i.e., SVRG method~(ODM$_{svrg}$)~\cite{Johnson2013SVRG} and CSVRG method~(ODM$_{csvrg}$)~\cite{Tan2019Coreset}.

\subsection{Results with RBF Kernel}

Figure~\ref{fig:1} presents the test accuracy and time cost of different methods with RBF kernel. It can be seen that SODM performs significantly better than other methods. Specifically, SODM achieves the best test accuracy on 7 data sets and just slightly worse than DC-ODM on data set \textsf{skin-nonskin}. As for time cost, SODM achieves the fastest training speed on all data sets. The detailed test accuracy and time cost are presented in Table~\ref{tab:2}. The time cost and test accuracy with corresponding SVM can be found in \nameref{appendixexp}.

\subsection{Results with Linear Kernel}

Figure~\ref{fig:2} presents the test accuracy and time cost of different methods with linear kernel. It can be seen that SODM shows highly competitive performance compared with other methods. Specifically, SODM achieves the best test accuracy on 6 data sets and just slightly worse than DC-ODM on data set \textsf{svmguide1} and \textsf{skin-nonskin}. As for time cost, SODM achieves faster training speed on all data sets. The detailed test accuracy and time cost are presented in Table~\ref{tab:3}. In Figure~\ref{fig:3}, we show the training speedup ratio with cores increasing from 1 to 32 for linear kernel and RBF kernel, respectively. When 32 cores used, RBF kernel SODM achieves more than 9 times training speedup while linear kernel SODM achieves over 5 times training speedup.

\subsection{Comparison with Gradient Based Methods}

Figure~\ref{fig:4} compares the test accuracy and time cost between our acceleration method and other gradient based methods. We observe that our method can get competitive result. Meanwhile, our method achieves over 5 times faster speed than other methods. This indicates that our scalable acceleration method can achieve great training speed while hold the generalization performance.

\section{Conclusion}

Although lots of works have been proposed to solve QP problems, these off-the-shelf solvers usually ignore the intrinsic structure of the optimization problem, thus can hardly achieve the greatest efficiency when directly applied to ODM. We propose a scalable ODM with a novel partition strategy, which can retain the first- and second- order statistics in both the original instance space and the RKHS, leading to significant speedup of training. In addition, an accelerating method is implemented to further improve the training when linear kernel is used. As shown in the experiments, SODM has great superiority to other scalable QP solvers in terms of both generalization performance and time cost. In the future, we will consider the circumstance in which data is located on different devices and can not be gathered together due to the limited bandwidth or user privacy.

\section*{Acknowledgments}

This work was supported in part by the National Key R\&D Program of China under Grant 2020AAA0108501, the National Natural Science Foundation of China under Grant 62006088, and the Key R\&D Program of Hubei under Grant 2020BAA020.

\bibliographystyle{named}
\bibliography{ijcai23}

\begin{thebibliography}{}

\bibitem[\protect\citeauthoryear{Boyd \bgroup \em et al.\egroup
  }{2010}]{Boyd2011Admm}
Stephen Boyd, Neal Parikh, Eric Chu, Borja Peleato, and Jonathan Eckstein.
\newblock Distributed {O}ptimization and {S}tatistical {L}earning via the
  {A}lternating {D}irection {M}ethod of {M}ultipliers.
\newblock {\em Foundations and Trends in Machine Learning}, 3(1):1--122, 2010.

\bibitem[\protect\citeauthoryear{Cao \bgroup \em et al.\egroup
  }{2006}]{Cao2006SMO}
Lijuan Cao, Selvaraj~Sathiya Keerthi, Chong~Jin Ong, Jianqiu Zhang, Uvaraj
  Periyathamby, Xiuju Fu, and Henry~P. Lee.
\newblock Parallel sequential minimal optimization for the training of support
  vector machines.
\newblock {\em IEEE Transactions on Neural Networks}, 17(4):1039--1049, 2006.

\bibitem[\protect\citeauthoryear{Cao \bgroup \em et al.\egroup
  }{2021}]{Cao2021Partial}
Nan Cao, Teng Zhang, and Hai Jin.
\newblock {Partial Multi-Label Optimal Margin Distribution Machine}.
\newblock In {\em Proceedings of the 30th International Joint Conference on
  Artificial Intelligence}, pages 2198--2204, Montreal-themed virtual reality,
  2021.

\bibitem[\protect\citeauthoryear{Cao \bgroup \em et al.\egroup
  }{2022}]{Cao2022Positive}
Nan Cao, Teng Zhang, Xuanhua Shi, and Hai Jin.
\newblock {Posistive-Unlabeled Learning via Optimal Transport and Margin
  Distribution}.
\newblock In {\em Proceedings of the 31st International Joint Conference on
  Artificial Intelligence}, pages 2836--2842, Vienna, Austria, 2022.

\bibitem[\protect\citeauthoryear{Cheng \bgroup \em et al.\egroup
  }{2017}]{Cheng2017Large}
Fanyong Cheng, Jing Zhang, Cuihong Wen, Zhaohua Liu, and Zuoyong Li.
\newblock Large cost-sensitive margin distribution machine for imbalanced data
  classification.
\newblock {\em Neurocomputing}, 224:45--57, 2017.

\bibitem[\protect\citeauthoryear{Forero \bgroup \em et al.\egroup
  }{2010}]{Forero2010Admm}
Pedro~A. Forero, Alfonso Cano, and Georgios~B. Giannakis.
\newblock Consensus-{B}ased {D}istributed {S}upport {V}ector {M}achines.
\newblock {\em Journal of Machine Learning Research}, 11:1663--1701, 2010.

\bibitem[\protect\citeauthoryear{Gao and Zhou}{2013}]{Gao2013On}
Wei Gao and Zhi-Hua Zhou.
\newblock On the doubt about margin explanation of boosting.
\newblock {\em Artificial Intelligence}, 203:1--18, 2013.

\bibitem[\protect\citeauthoryear{Graf \bgroup \em et al.\egroup
  }{2004}]{Graf2004cascade}
Hans~Peter Graf, Eric Cosatto, Leon Bottou, Igor Dourdanovic, and Vladimir
  Vapnik.
\newblock Parallel {S}upport {V}ector {M}achines: {T}he {C}ascade {SVM}.
\newblock In {\em Advances in Neural Information Processing Systems}, pages
  521--528, Vancouver, Canada, 2004.

\bibitem[\protect\citeauthoryear{Grønlund \bgroup \em et al.\egroup
  }{2019}]{Allan2019Margin}
Allan Grønlund, Lior Kamma, Kasper~Green Larsen, Alexander Mathiasen, and
  Jelani Nelson.
\newblock Margin-{B}ased {G}eneralization {L}ower {B}ounds for {B}oosted
  {C}lassifiers.
\newblock In {\em Advances in Neural Information Processing Systems}, pages
  11963--11972, Vancouver, Canada, 2019.

\bibitem[\protect\citeauthoryear{Hsieh \bgroup \em et al.\egroup
  }{2014}]{Hsieh2014dc}
Cho-Jui Hsieh, Si~Si, and Inderjit~Singh Dhillon.
\newblock A {D}ivide-and-{C}onquer {S}olver for {K}ernel {S}upport {V}ector
  {M}achines.
\newblock In {\em Proceedings of the 31st International Conference on Machine
  Learning}, pages 566--574, Beijing, China, 2014.

\bibitem[\protect\citeauthoryear{Johnson and Zhang}{2013}]{Johnson2013SVRG}
Rie Johnson and Tong Zhang.
\newblock Accelerating {S}tochastic {G}radient {D}escent using {P}redictive
  {V}ariance {R}eduction.
\newblock In {\em Advances in Neural Information Processing Systems}, pages
  315--323, Lake Tahoe, NV, 2013.

\bibitem[\protect\citeauthoryear{Lee \bgroup \em et al.\egroup
  }{2017}]{Lee2017DSVRG}
Jason~D. Lee, Qihang Lin, Tengyu Ma, and Tianbao Yang.
\newblock Distributed {S}tochastic {V}ariance {R}educed {G}radient {M}ethods by
  {S}ampling {E}xtra {D}ata with {R}eplacement.
\newblock {\em Journal of Machine Learning Research}, 18(122):1--43, 2017.

\bibitem[\protect\citeauthoryear{Loosli \bgroup \em et al.\egroup
  }{2007}]{loosli2007training}
Ga{\"e}lle Loosli, St{\'e}phane Canu, and L{\'e}on Bottou.
\newblock Training invariant support vector machines using selective sampling.
\newblock In L\'{e}on Bottou, Olivier Chapelle, Dennis {DeCoste}, and Jason
  Weston, editors, {\em Large-Scale Kernel Machines}, pages 301--320. MIT
  Press, Cambridge, MA, 2007.

\bibitem[\protect\citeauthoryear{Luan \bgroup \em et al.\egroup
  }{2020}]{Luan2020Optimal}
Tianxiang Luan, Tingjin Luo, Wenzhang Zhuge, and Chenping Hou.
\newblock Optimal {R}epresentative {D}istribution {M}argin {M}achine for
  {M}ulti-{I}nstance {L}earning.
\newblock {\em IEEE Access}, 8:74864--74874, 2020.

\bibitem[\protect\citeauthoryear{Navia-Vazquez \bgroup \em et al.\egroup
  }{2006}]{Vazquez2006dsvm}
Angel Navia-Vazquez, D.~Gutierrez-Gonzalez, Emilio Parrado-Hernandez, and J.~J.
  Navarro-Abellan.
\newblock Distributed {S}upport {V}ector {M}achines.
\newblock {\em IEEE Transactions on Neural Networks}, 17(4):1091–1097, 2006.

\bibitem[\protect\citeauthoryear{Rahimi and Recht}{2007}]{Rahimi2007Random}
Ali Rahimi and Benjamin Recht.
\newblock Random {F}eatures for {L}arge-{S}cale {K}ernel {M}achines.
\newblock In {\em Advances in Neural Information Processing Systems}, pages
  1177--1184, Vancouver, Canada, 2007.

\bibitem[\protect\citeauthoryear{Rastogi \bgroup \em et al.\egroup
  }{2020}]{Rastogi2020Large}
Reshma Rastogi, Pritam Anand, and Suresh Chandra.
\newblock Large-margin {D}istribution {M}achine-based regression.
\newblock {\em Neural Computing and Applications}, 32:3633–3648, 2020.

\bibitem[\protect\citeauthoryear{Schölkopf and
  Smola}{2001}]{Scholkopf2001Learning}
Bernhard Schölkopf and Alexander~Johannes Smola.
\newblock {\em Learning with kernels: support vector machines, regularization,
  optimization, and beyond}.
\newblock MIT Press, Cambridge, MA, 2001.

\bibitem[\protect\citeauthoryear{Si \bgroup \em et al.\egroup
  }{2017}]{Si2017MEKA}
Si~Si, Cho-Jui Hsieh, and Inderjit~Singh Dhillon.
\newblock {Memory Efficient Kernel Approximation}.
\newblock {\em Journal of Machine Learning Research}, 18:1--32, 2017.

\bibitem[\protect\citeauthoryear{Singh \bgroup \em et al.\egroup
  }{2017}]{Singh2017Dip}
Dinesh Singh, Debaditya Roy, and Chalavadi~Krishna Mohan.
\newblock Di{P}-{SVM}: Distribution {P}reserving {K}ernel {S}upport {V}ector
  {M}achine for {B}ig {D}ata.
\newblock {\em IEEE Transactions on Big Data}, 3(1):79--90, 2017.

\bibitem[\protect\citeauthoryear{Tan \bgroup \em et al.\egroup
  }{2019}]{Tan2019Coreset}
Zhi-Hao Tan, Teng Zhang, and Wei Wang.
\newblock Coreset {S}tochastic {V}ariance-{R}educed {G}radient with
  {A}pplication to {O}ptimal {M}argin {D}istribution {M}achine.
\newblock In {\em Proceedings of the 33rd AAAI Conference on Artificial
  Intelligence}, pages 5083--5090, Honolulu, HI, 2019.

\bibitem[\protect\citeauthoryear{Tan \bgroup \em et al.\egroup
  }{2020}]{Tan2020Multi}
Zhi-Hao Tan, Peng Tan, Yuan Jiang, and Zhi-Hua Zhou.
\newblock Multi-label {O}ptimal {M}argin {D}istribution {M}achine.
\newblock {\em Machine Learning}, 109(3):623--642, 2020.

\bibitem[\protect\citeauthoryear{Williams and Seeger}{2001}]{Williams2001Using}
Christopher Williams and Matthias Seeger.
\newblock Using the {N}yström {M}ethod to {S}peed {U}p {K}ernel {M}achines.
\newblock In {\em Advances in Neural Information Processing Systems}, pages
  682--688, Cambridge, MA, 2001.

\bibitem[\protect\citeauthoryear{Yu \bgroup \em et al.\egroup
  }{2005}]{Yu2005Making}
Hwanjo Yu, Jiong Yang, Jiawei Han, and Xiaolei Li.
\newblock Making {SVM}s scalable to large data sets using hierarchical cluster
  indexing.
\newblock {\em Data Mining and Knowledge Discovery}, 11(3):295--321, 2005.

\bibitem[\protect\citeauthoryear{Zaharia \bgroup \em et al.\egroup
  }{2012}]{Zaharia2012Spark}
Matei Zaharia, Mosharaf Chowdhury, Tathagata Das, Ankur Dave, Justin Ma, Murphy
  McCauly, Michael~Jay Franklin, Scott Shenker, and Ion Stoica.
\newblock Resilient {D}istributed {D}atasets: A {F}ault-{T}olerant
  {A}bstraction for {I}n-{M}emory {C}luster {C}omputing.
\newblock In {\em Proceedings of the 9th {USENIX} Symposium on Networked
  Systems Design and Implementation}, pages 15--28, San Jose, CA, 2012.

\bibitem[\protect\citeauthoryear{Zhang and Jin}{2020}]{Zhang2020Instance}
Teng Zhang and Hai Jin.
\newblock Optimal {M}argin {D}istribution {M}achine for {M}ulti-{I}nstance
  {L}earning.
\newblock In {\em Proceedings of the 29th International Joint Conference on
  Artificial Intelligence}, pages 2383--2389, 2020.

\bibitem[\protect\citeauthoryear{Zhang and Zhou}{2018a}]{Zhang2018Optimal}
Teng Zhang and Zhi-Hua Zhou.
\newblock Optimal margin distribution clustering.
\newblock In {\em Proceedings of the 32nd AAAI Conference on Artificial
  Intelligence}, pages 4474--4481, New Orleans, LA, 2018.

\bibitem[\protect\citeauthoryear{Zhang and Zhou}{2018b}]{Zhang2018Semi}
Teng Zhang and Zhi-Hua Zhou.
\newblock Semi-supervised optimal margin distribution machines.
\newblock In {\em Proceedings of the 27th International Joint Conference on
  Artificial Intelligence}, pages 3104--3110, Stockholm, Sweden, 2018.

\bibitem[\protect\citeauthoryear{Zhang and Zhou}{2019}]{Zhang2019Optimal}
Teng Zhang and Zhi-Hua Zhou.
\newblock Optimal {M}argin {D}istribution {M}achine.
\newblock {\em IEEE Transactions on Knowledge and Data Engineering},
  32(6):1143--1156, 2019.

\bibitem[\protect\citeauthoryear{Zhang \bgroup \em et al.\egroup
  }{2020}]{Zhang2020Dynamic}
Teng Zhang, Peng Zhao, and Hai Jin.
\newblock Optimal {M}argin {D}istribution {L}earning in {D}ynamic
  {E}nvironments.
\newblock In {\em Proceedings of the 34th AAAI Conference on Artificial
  Intelligence}, pages 6821--6828, New York, NY, 2020.

\bibitem[\protect\citeauthoryear{Zhou and Zhou}{2016}]{Zhou2016Large}
Yu-Hang Zhou and Zhi-Hua Zhou.
\newblock Large {M}argin {D}istribution {L}earning with {C}ost {I}nterval and
  {U}nlabeled data.
\newblock {\em IEEE Transactions on Knowledge and Data Engineering},
  28(7):1749--1763, 2016.

\end{thebibliography}

\newpage
\appendix
\onecolumn

\section{Theoretical Proof}
\label{appendixproof}
\setcounter{theorem}{0}
In this section, we first infer the formulation of SODM in details. Then we give the full proof of Theorem~1 and Theorem~2.
\subsection{Preliminaries}

Given a labeled data set $\{ (\xv_i, y_i) \}_{i \in [M]}$, the primal problem of ODM is
\begin{align*}
    \min_{\wv, \xi_i, \epsilon_i} ~ p(\wv) = \frac{1}{2} \|\wv\|^2 + \frac{\lambda}{2M} \sum_{i \in [M]} \frac{\xi_i^2 + \upsilon \epsilon_i^2}{(1 - \theta)^2}, \quad \st ~ 1 - \theta - \xi_i \le y_i \wv^\top \phi(\xv_i) \le 1 + \theta + \epsilon_i, ~ \forall i \in [M].
\end{align*}
Denote $\Xv = [\phi(\xv_1), \ldots, \phi(\xv_M)]$, $\Yv = \diag(y_1, \ldots, y_M)$, $\xiv = [\xi_1; \ldots; \xi_M]$, $\epsilonv = [\epsilon_1; \ldots; \epsilon_M]$, the above formulation can be rewritten as
\begin{align} \label{eq: odm}
    \min_{\wv, \xiv, \epsilonv} ~ p(\wv) = \frac{1}{2} \|\wv\|^2 + \frac{\lambda (\|\xiv\|^2 + \upsilon \|\epsilonv\|^2)}{2M (1 - \theta)^2}, \quad \st ~ (1 - \theta) \onev_M - \xiv \le \Yv \Xv^\top \wv \le (1 + \theta) \onev_M + \epsilonv,
\end{align}
where $\onev_M$ is the $M$-dimensional all one vector.

With Lagrange multipliers $\zetav, \betav \in \Rbb^M_+$ for the two constraints respectively, the Lagrangian of Eqn.~\eqref{eq: odm} leads to
\begin{align} \label{eq: Lagrangian}
    \begin{split}
        L = \frac{1}{2} \|\wv\|^2 + \frac{\lambda (\|\xiv\|^2 + \upsilon \|\epsilonv\|^2)}{2M (1 - \theta)^2} - \zetav^\top (\Yv \Xv^\top \wv - (1 - \theta) \onev_M + \xiv) + \betav^\top (\Yv \Xv^\top \wv - (1 + \theta) \onev_M - \epsilonv),
    \end{split}
\end{align}
and the KKT conditions are
\begin{align}
    \label{eq: kkt-1}
     & \wv = \Xv  \Yv  (\zetav - \betav), \quad \xiv = \frac{M (1 - \theta)^2}{\lambda} \zetav, \quad \epsilonv = \frac{M (1 - \theta)^2}{\lambda \upsilon} \betav,      \\
    \label{eq: kkt-2}
     & \zeta_i (y_i \wv^\top \phi(\xv_i) - (1 - \theta) + \xi_i) = 0, \quad \beta_i (y_i \wv^\top \phi(\xv_i) - (1 + \theta) - \epsilon_i) = 0, \quad \forall i \in [M].
\end{align}
Eqn.~\eqref{eq: kkt-1} is derived by setting the partial derivative of $L$ w.r.t. $\{\wv, \xiv, \epsilonv\}$ to zero. Eqn.~\eqref{eq: kkt-2} is the complementary slackness conditions. Observe that $y_i \wv^\top \phi(\xv_i) < 1 - \theta$ and $y_i \wv^\top \phi(\xv_i) > 1 + \theta$ cannot hold simultaneously, therefore at least one of the two slack variables $\xi_i$ and $\epsilon_i$ is zero. According to Eqn.~\eqref{eq: kkt-1}, we have $\zeta_i \beta_i = 0$ for any $i \in [M]$.

The following dual problem of ODM follows by substituting Eqn.~\eqref{eq: kkt-1} back into Eqn.~\eqref{eq: Lagrangian}:
\begin{align} \label{eq: dual-odm}
    \min_{\zetav, \betav \in \Rbb^M_+} & ~d(\zetav, \betav) = \frac{1}{2} (\zetav - \betav)^\top \Qv (\zetav - \betav) + \frac{M c}{2} (\upsilon \|\zetav\|^2 + \|\betav\|^2) + (\theta-1) \onev_M^\top \zetav + (\theta+1) \onev_M^\top \betav,
\end{align}
where $\Qv = \Yv \Xv^\top \Xv \Yv$ and $c = (1 - \theta)^2 / \lambda \upsilon$ is a constant. By denoting $\alphav = [\zetav; \betav]$, above problem can be rewritten as a standard convex quadratic programming:
\begin{align*}
    \min_{\alphav \in \Rbb^{2M}_+} ~ f(\alphav) = \frac{1}{2} \alphav^\top
    \begin{bmatrix}
        \Qv + M c \upsilon \Iv & -\Qv          \\
        - \Qv                  & \Qv + M c \Iv
    \end{bmatrix} \alphav +
    \begin{bmatrix}
        (\theta - 1) \onev_M \\
        (\theta + 1) \onev_M
    \end{bmatrix}^\top \alphav.
\end{align*}

Suppose $\{ (\xv_i^{(k)}, y_i^{(k)}) \}_{i \in [m]}$ are the instances in the $k$-th partition, then the dual problem of ODM on the $k$-th partition is
\begin{align*}
    \min_{\zetav_k, \betav_k \in \Rbb^m_+} d_k(\zetav_k, \betav_k) = \frac{1}{2} (\zetav_k - \betav_k)^\top \Qv^{(k)} (\zetav_k - \betav_k) + \frac{m c}{2} ( \upsilon \|\zetav_k\|^2 + \|\betav_k\|^2) + (\theta - 1) \onev_m^\top \zetav_k + (\theta + 1) \onev_m^\top \betav_k,
\end{align*}
where $\Qv_k = \Yv_k {\Xv_k}^\top \Xv_k \Yv_k$, $\Xv_k = [\phi(\xv_1^{(k)}), \ldots, \phi(\xv_m^{(k)})]$, and $\Yv_k = \diag(y_1^{(k)}, \ldots, y_m^{(k)})$. Notice that the optimization variables $\zetav_k$ and $\betav_k$ are decoupled on each partition, by merging all the $K$ problems together, we can get the formulation of SODM:
\begin{align*}
    \min_{\zetav, \betav \in \Rbb^M_+} & ~\dt (\zetav, \betav) = \frac{1}{2} (\zetav - \betav)^\top \Qvt (\zetav - \betav) + \frac{m c}{2} (\upsilon \|\zetav\|^2 + \|\betav\|^2) + (\theta-1) \onev_M^\top \zetav + (\theta+1) \onev_M^\top \betav,
\end{align*}
where $\Qvt = \diag(\Qv_1, \ldots, \Qv_k)$ is a block diagonal matrix, $\zetav = [\zetav_1; \ldots; \zetav_k]$, and $\betav = [\betav_1; \ldots; \betav_k]$.

\subsection{Proof of Theorem~1}

\begin{theorem} \label{thm-1}
    Suppose the optimal solutions of ODM and SODM are $\alphav^\star = [\zetav^\star; \betav^\star]$ and $\alphavt^\star = [\zetavt^\star; \betavt^\star]$, respectively. The gaps between the optimal objective values and solutions satisfy
    \begin{align}
        \label{eq: bound-1}
         & 0 \leq d(\zetavt^\star, \betavt^\star) - d(\zetav^\star, \betav^\star) \leq U^2 (Q + M (M - m) c), \\
        \label{eq: bound-2}
         & \| \alphavt^\star - \alphav^\star\|^2 \le \frac{U^2}{M c \upsilon} (Q + M (M - m) c),
    \end{align}
    where $U = \max ( \|\alphav^\star\|_\infty, \|\alphavt^\star\|_\infty)$ and $Q = \sum_{i,j:P(i) \neq P(j)} |[\Qv]_{ij}|$.
\end{theorem}

\begin{proof}[Proof]
    The left-hand side of Eqn.~\eqref{eq: bound-1} is due to the optimality of $\zetav^\star$ and $\betav^\star$.

    Without loss of generality, suppose $\{ (\xv_i, y_i) \}_{i \in [M]}$ are ordered by partition index, i.e., the first $m$ instances are on the first partition, and the second $m$ instances are on the second partition, etc. According to the definition of $d(\zetav, \betav)$ and $\dt (\zetav, \betav)$, and by denoting $\gammav = \zetav - \betav$, we have
    \begin{align*}
        d(\zetav, \betav) & = \dt (\zetav, \betav) + \frac{1}{2} (\zetav - \betav)^\top (\Qv - \Qvt) (\zetav - \betav) + \frac{(M - m) c}{2} (\upsilon \|\zetav\|^2 + \|\betav\|^2)         \\
                          & = \dt (\zetav, \betav) + \frac{1}{2} \sum_{i,j: P(i) \ne P(j)} [\gammav]_i [\gammav]_j [\Qv]_{ij} + \frac{(M - m) c}{2} (\upsilon \|\zetav\|^2 + \|\betav\|^2).
    \end{align*}
    In particular, the following holds:
    \begin{align}
        \label{eq: eqn-1}
        d(\zetav^\star, \betav^\star)   & = \dt (\zetav^\star, \betav^\star) + \frac{1}{2} \sum_{i,j: P(i) \ne P(j)} [\gammav^\star]_i [\gammav^\star]_j [\Qv]_{ij} + \frac{(M - m) c}{2} (\upsilon \|\zetav^\star\|^2 + \|\betav^\star\|^2),       \\
        \label{eq: eqn-2}
        d(\zetavt^\star, \betavt^\star) & = \dt (\zetavt^\star, \betavt^\star) + \frac{1}{2} \sum_{i,j: P(i) \ne P(j)} [\gammavt^\star]_i [\gammavt^\star]_j [\Qv]_{ij} + \frac{(M - m) c}{2} (\upsilon \|\zetavt^\star\|^2 + \|\betavt^\star\|^2).
    \end{align}
    Notice that at least one of $\zeta_i$ and $\beta_i$ is zero, thus $|\gamma_i| \le |\zeta_i| + |\beta_i| = \max (|\zeta_i|, |\beta_i|) \le U$. Subtracting Eqn.~\eqref{eq: eqn-1} from Eqn.~\eqref{eq: eqn-2} yields the right-hand side of Eqn.~\eqref{eq: bound-1}:
    \begin{align*}
        d(\zetavt^\star, \betavt^\star) - d(\zetav^\star, \betav^\star)
         & \le \frac{1}{2} \sum_{i,j: P(i) \ne P(j)} |[\gammavt^\star]_i [\gammavt^\star]_j - [\gammav^\star]_i [\gammav^\star]_j| \cdot |[\Qv]_{ij}| + \frac{(M - m) c}{2} (\upsilon \|\zetav^\star\|^2 + \|\betav^\star\|^2) \\
         & \le U^2 Q + U^2 M (M - m) c,
    \end{align*}
    where the first inequality follows from the optimality of $\zetavt^\star$ and $\betavt^\star$, and the third inequality is derived by the boundness of $\zetav, \betav, \gammav$ and $\upsilon \le 1$.

    Since $f(\alphav)$ is a quadratic function, it can be expanded at $\alphav^\star$ as
    \begin{align*}
        f(\alphavt^\star)
         & = f(\alphav^\star) + \nabla f(\alphav^\star)^\top (\alphavt^\star - \alphav^\star) + \frac{1}{2}(\alphavt^\star - \alphav^\star)^\top \nabla^2 f(\alphav^\star) (\alphavt^\star - \alphav^\star) \\
         & \ge f(\alphav^\star) + \frac{1}{2}(\alphavt^\star - \alphav^\star)^\top \nabla^2 f(\alphav^\star) (\alphavt^\star - \alphav^\star)                                                               \\
         & = f(\alphav^\star) + \frac{1}{2}(\alphavt^\star - \alphav^\star)^\top
        \begin{bmatrix}
            \Qv + M c \upsilon \Iv & -\Qv          \\
            - \Qv                  & \Qv + M c \Iv
        \end{bmatrix} (\alphavt^\star - \alphav^\star)                                                                                                                                           \\
         & \ge  f(\alphav^\star) + \frac{1}{2}(\alphavt^\star - \alphav^\star)^\top
        \begin{bmatrix}
            M c \upsilon \Iv \\
             & M c \Iv
        \end{bmatrix} (\alphavt^\star - \alphav^\star)                                                                                                                                           \\
         & \ge  f(\alphav^\star) + \frac{1}{2} M c \upsilon \| \alphavt^\star - \alphav^\star \|^2,
    \end{align*}
    where the first inequality follows from the first order optimal condition, and the third inequality uses the fact $\upsilon \le 1$. Thus $\| \alphavt^\star - \alphav^\star \|^2$ can be upper bounded by
    \begin{align*}
        \| \alphavt^\star - \alphav^\star \|^2 \le \frac{1}{M c \upsilon} (d(\zetavt^\star, \betavt^\star) - d(\zetav^\star, \betav^\star)) \le \frac{U^2}{M c \upsilon} (Q + M (M - m) c),
    \end{align*}
    which shows that Eqn.~\eqref{eq: bound-2} holds and concludes the proof.
\end{proof}

\subsection{Proof of Theorem~2}

\begin{theorem}
    For shift-invariant kernel $\kappa$ with $\kappa(0) = r^2$, that is $\|\phi(\xv)\| = r$ for any $\xv$. With the partition strategy described above, for any $k \in [K]$, we have
    \begin{align} \label{eq: bound-3}
        d_k(\zetav_k, \betav_k) - d(\zetav^\star, \betav^\star) \leq \frac{U^2}{2} ( M^2 r^2 + r^2 \cos \tau (2C - M^2) ) + U^2 M^2 c +2 U M,
    \end{align}
    where $C = \sum_{i,j \in [M]} \Ibb (\varphi(\xv_i) \neq \varphi(\xv_j))$, and $U$ is defined in theorem~\ref{thm-1}.
\end{theorem}

\begin{proof}[Proof]
    Construct the auxiliary data set $\Dcalt_k$ by repeating each instance in $\Dcal_k$ for $K$ times, i.e.,
    \begin{align*}
        \Dcalt_k = \{ \underbrace{(\xv_1^{(k)}, y_1^{(k)}), \ldots ,(\xv_m^{(k)}, y_m^{(k)})}_{1}, \underbrace{(\xv_1^{(k)}, y_1^{(k)}), \ldots ,(\xv_m^{(k)}, y_m^{(k)})}_{2}, \ldots, \underbrace{(\xv_1^{(k)}, y_1^{(k)}), \ldots ,(\xv_m^{(k)}, y_m^{(k)})}_{K} \}.
    \end{align*}
    It can be seen that primal ODM on $\Dcalt_k$ and $\Dcal_k$ have the same constraints (by removing repetitions), thus for any $\wv$, it is feasible on $\Dcalt_k$ iff it is feasible on $\Dcal_k$. In addition, we have
    \begin{align*}
        \pt_k(\wv) = \frac{1}{2} \|\wv\|^2 + \frac{\lambda}{2M} \sum_{i \in [M]} \frac{K (\xi_i^2 + \upsilon \epsilon_i^2)}{(1 - \theta)^2} = \frac{1}{2} \|\wv\|^2 + \frac{\lambda}{2m} \sum_{i \in [m]} \frac{\xi_i^2 + \upsilon \epsilon_i^2}{(1 - \theta)^2} = p_k(\wv).
    \end{align*}
    Therefore, primal ODM on $\Dcalt_k$ and $\Dcal_k$ have the same optimal objective. Since strong dual theorem holds for ODM, we have $d_k(\zetav_k, \betav_k) = p_k(\wv_k) = \pt_k(\wv) = \dt_k (\zetavt_k, \betavt_k)$. Notice that [cf. \eqref{eq: dual-odm}]
    \begin{align*}
        \dt_k (\zetavt_k, \betavt_k) = \frac{1}{2} (\zetavt_k - \betavt_k)^\top \Qvt_k (\zetavt_k - \betavt_k) + \frac{M c}{2} (\upsilon \|\zetavt_k\|^2 + \|\betavt_k\|^2) + (\theta - 1) \onev_M^\top \zetavt_k + (\theta + 1) \onev_M^\top \betavt_k,
    \end{align*}
    where $\Qvt_k = \onev_{K \times K} \otimes \Qv_k$. Then the left-hand side of Eqn.~\eqref{eq: bound-3} is
    \begin{align*}
        d_k(\zetav_k , \betav_k) - d(\zetav^\star, \betav^\star) = \dt_k(\zetavt_k, \betavt_k) - d(\zetav^\star, \betav^\star) = A + B,
    \end{align*}
    where
    \begin{align*}
        A & = \frac{1}{2} (\zetavt_k - \betavt_k)^\top \Qvt_k (\zetavt_k - \betavt_k) - \frac{1}{2} (\zetav^\star - \betav^\star)^\top \Qv (\zetav^\star - \betav^\star),                                                                         \\
        B & = \frac{M c \upsilon}{2} (\|\zetavt_k\|^2 - \|\zetav^\star\|^2) + \frac{M c}{2} (\|\betavt_k\|^2 - \|\betav^\star\|^2) + (\theta - 1) \onev_M^\top (\zetavt_k - \zetav^\star) + (\theta + 1) \onev_M^\top (\betavt_k - \betav^\star).
    \end{align*}
    Notice that $0 \leq \upsilon \leq 1$, $0 \leq \theta \leq 1$ and $\zetav, \betav$ is upper bounded by $U$, we have
    \begin{align*}
        B \le \frac{M c \upsilon}{2} \|\zetavt_k\|^2 + \frac{M c}{2} \|\betavt_k\|^2 + (\theta - 1) \onev_M^\top (\zetavt_k - \zetav^\star) + (\theta + 1) \onev_M^\top (\betavt_k - \betav^\star) \le U^2 M^2 c + 2 U M.
    \end{align*}
    As for the first term, by denoting $\gammav = \zetav - \betav$, it can be seen that
    \begin{align*}
        A = \frac{1}{2} \gammavt_k^\top \Qvt_k \gammavt_k - \frac{1}{2} \gammavt^\star {}^\top \Qv \gammavt^\star \le \frac{U^2}{2} \sum_{i,j \in [M]} | [\Qvt_k]_{ij} - [\Qv]_{ij} |.
    \end{align*}
    Suppose the angle between $\phi(\xv_i)$ and $\phi(\xv_j)$ is $\vartheta$, then
    \begin{align*}
        [\Qv]_{ij} = \phi(\xv_i)^\top \phi(\xv_j) = \| \phi(\xv_i) \| \| \phi(\xv_j) \| \cos \vartheta = r^2 \cos \vartheta \in \begin{cases}
                                                                                                                                    [-r^2, r^2 \cos \tau], & \varphi(\xv_i) \ne \varphi(\xv_j), \\
                                                                                                                                    [r^2 \cos \tau, r^2],  & \varphi(\xv_i) = \varphi(\xv_j).
                                                                                                                                \end{cases}
    \end{align*}
    The arguments for $[\Qvt_k]_{ij}$ is similar and we have
    \begin{align*}
        A & \le \frac{U^2}{2} \left( \sum_{\varphi(\xv_i) \ne \varphi(\xv_j)} | [\Qvt_k]_{ij} - [\Qv]_{ij} | + \sum_{\varphi(\xv_i) = \varphi(\xv_j)} | [\Qvt_k]_{ij} - [\Qv]_{ij} | \right) \\
          & \le \frac{U^2}{2} \left( \sum_{\varphi(\xv_i) \ne \varphi(\xv_j)} r^2 (1 + \cos \tau) + \sum_{\varphi(\xv_i) = \varphi(\xv_j)} r^2 (1 - \cos \tau) \right)                       \\
          & = \frac{U^2}{2} ( C r^2 (1 + \cos \tau) + (M^2 - C) r^2 (1 - \cos \tau) )                                                                                                        \\
          & = \frac{U^2}{2} ( M^2 r^2 + r^2 \cos \tau (2C - M^2) ).
    \end{align*}
    By putting the upper bound of $A$ and $B$ together concludes the proof.
\end{proof}
\newpage
\section{Supplementary Experiments}
\label{appendixexp}
In this section, We supplement the experiments of Scalable SVM. Here we compared the results of Ca-ODM, DiP-ODM, DC-ODM and SODM with corresponding SVM methods on all datasets using rbf kernel as supplementary.

\begin{table*}[!h]
    \centering
    \begin{tabular}{ l | c | c | c | c | c | c | c | c }
        \hline \hline \noalign{\smallskip}
        \multirow{2}*{Data sets} & \multicolumn{2}{c|}{Ca-SVM} & \multicolumn{2}{c|}{Ca-ODM}&  \multicolumn{2}{c|}{Dip-SVM} & \multicolumn{2}{c}{Dip-ODM}  \\
        \noalign{\smallskip} \cline{2-9} \noalign{\smallskip}
        ~           & Acc.  & Time(s) & Acc. & Time(s) & Acc. & Time(s) & Acc. & Time(s)\\
        \noalign{\smallskip} \hline \noalign{\smallskip}
        gisette  & .932 & 104.67  & .957 & 90.22   & .925 & 67.98 & .970 & 68.02 
        \\
        svmguide1  & .904 & 49.20  & .872 & 38.90  & .895  & 33.20 & .903  & 35.25 
        \\
        phishing & .910  & 43.85 & .880  & 49.60  & .902  & 55.02  & .901  & 52.61 
        \\
        a7a & .817  & 59.40  & .824  & 68.36  & .815  & 58.92 & .813  & 61.24 
        \\
        cod-rna  & .880  & 458.43  & .892  & 499.38  & .873  & 508.33  & .905  & 532.68 
        \\
        ijcnn1  & .803  & 150.11  & .889  & 185.20  & .824  & 156.27  & .893  & 182.71 
        \\
        skin-nonskin  & .811  & 299.96 & .806  & 338.73 & .855  & 343.82 & .830  & 437.20
        \\
        SUSY  & .720  & 4153.10  & .733  & 4280.23  & .752  & 5377.99  & .744  & 5678.66
        \\
        \noalign{\smallskip} \hline \hline \noalign{\smallskip}
    \multirow{2}*{Data sets} & \multicolumn{2}{c|}{DC-SVM} & \multicolumn{2}{c|}{DC-ODM}&  \multicolumn{2}{c|}{SSVM} & \multicolumn{2}{c}{SODM}  \\
        \noalign{\smallskip} \cline{2-9} \noalign{\smallskip}
        ~           & Acc.  & Time(s) & Acc. & Time(s) & Acc. & Time(s) & Acc. & Time(s)\\
        \noalign{\smallskip} \hline \noalign{\smallskip}
        gisette & .966 & 72.50  & .964 & 70.44  & .948 & 53.32 & .972 & 59.89 
        \\
        svmguide1  & .952 & 37.63 & .943 & 50.11  & .902 & 20.33 & .944 & 28.74
        \\
        phishing & .928 & 42.53 & .936 & 59.47 & .929  & 30.70  & .938  & 25.22 
        \\
        a7a & .818 & 97.99 & .815 & 106.51 & .810  & 40.54 & .838  & 32.67 
        \\
        cod-rna  & .915 & 430.26 & .931 & 400.61 & .889  & 64.91 & .933  & 55.41 
        \\
        ijcnn1   & .920 & 266.95 & .915 & 226.26 & .803 & 105.11 & .927  & 40.32 
        \\
        skin-nonskin   & .959 & 420.51 & .962 & 407.46 & .848  & 320.05 & .956  & 283.36  
        \\
        SUSY   & .758 & 7520.00 & .747 & 7009.36 & .720  & 3920.28 & .760  & 1004.33 
        \\
        \noalign{\smallskip} \hline \hline
    \end{tabular}
    \caption{The test accuracy and time cost of different methods using RBF kernel.}
    \label{tab:1}
\end{table*}

We conclude the detailed test accuracy and time cost in Table~\ref{tab:1} and get the following observation. 
\begin{itemize}
\item It can be seen that DC-SVM performs significantly better generality than other SVM methods. Specifically, DC-SVM achieves the best test accuracy on 6 data sets among SVM methods, just slightly worse than DC-SVM on phishing dataset and worse than Ca-SVM on a7a dataset. On time cost, SSVM achieves the fastest training speed on all 7 data sets and worse than Ca-SVM on skin-nonskin dataset.

\item Compared with SODM, SSVM achieves worse test accuracy on all 8 datasets and lower training speed on 7 datasets. Since SODM considers margin distribution by partitioning data and making local data distribution similar with the global one, it is more suitable for this task.

\item Compared with DC-ODM, DC-SVM performs better test accuracy on 5 datasets. Besides, the training time of these two methods are closed since they have the same parallel mechanism.  

\item Compared with Ca-ODM and Dip-ODM, the corresponding SVM methods achieves better time efficiency. Ca-SVM achieves better time efficiency on 6 datasets, while Dip-SVM achieves better time efficiency on 7 datasets. Since these two methods greedily discard data during optimization. On generality, Ca-SVM outperforms Ca-ODM on 3 datasets, while Dip-SVM outperforms Dip-ODM on 4 datasets.

\end{itemize}

\end{document}